\documentclass[12pt]{article}%
   \usepackage{fullpage}%

   \usepackage{titlesec}%
   \titlelabel{\thetitle. }%

\IfFileExists{sariel_computer.sty}{\def\sarielComp{1}}{}
\ifx\sarielComp\undefined%
\newcommand{\SarielComp}[1]{}
\newcommand{\NotSarielComp}[1]{#1}%
\else
\newcommand{\SarielComp}[1]{#1}%
\newcommand{\NotSarielComp}[1]{}%
\fi
\newcommand{\IfPrinterVer}[2]{#2}%

\usepackage{amsmath}%
\usepackage{amssymb}%
\usepackage{xcolor}%
\usepackage{scalerel}%
\usepackage{euscript}%
\usepackage{mathabx}%

\SarielComp{\usepackage{sariel_colors}}%

\usepackage[amsmath,thmmarks]{ntheorem}%
\theoremseparator{.}%

\usepackage{mleftright}%
\usepackage{xspace}%
\usepackage{hyperref}%

\newcommand{\hrefb}[3][black]{\href{#2}{\color{#1}{#3}}}%

\IfPrinterVer{%
   \usepackage{hyperref}%
}{%
   \usepackage{hyperref}%
   \hypersetup{%
      breaklinks,%
      ocgcolorlinks, colorlinks=true,%
      urlcolor=[rgb]{0.25,0.0,0.0},%
      linkcolor=[rgb]{0.5,0.0,0.0},%
      citecolor=[rgb]{0,0.2,0.445},%
      filecolor=[rgb]{0,0,0.4},
      anchorcolor=[rgb]={0.0,0.1,0.2}%
   }
}

\theoremseparator{.}%

\theoremstyle{plain}%
\newtheorem{theorem}{Theorem}[section]

\newtheorem{lemma}[theorem]{Lemma}

\newtheorem{observation}[theorem]{Observation}

\theoremstyle{plain}%
\theoremheaderfont{\sf} \theorembodyfont{\upshape}%
\newtheorem*{remark:unnumbered}[theorem]{Remark}%
\newtheorem{remark}[theorem]{Remark}%
\newtheorem{definition}[theorem]{Definition}
\newtheorem{defn}[theorem]{Definition}

\newcommand{\myqedsymbol}{\rule{2mm}{2mm}}

\theoremheaderfont{\em}%
\theorembodyfont{\upshape}%
\theoremstyle{nonumberplain}%
\theoremseparator{}%
\theoremsymbol{\myqedsymbol}%
\newtheorem{proof}{Proof:}%

\newcommand{\atgen}{\symbol{'100}}
\newcommand{\SarielThanks}[1]{\thanks{Department of Computer Science;
      University of Illinois; 201 N. Goodwin Avenue; Urbana, IL,
      61801, USA; {\tt sariel\atgen{}illinois.edu}; {\tt
         \url{http://sarielhp.org/}.} #1}}
\newcommand{\SepidehThanks}[1]{%
   \thanks{%      
      \texttt{mahabadi@ttic.edu}.%
      #1
   }%
}

\newcommand{\HLink}[2]{\hyperref[#2]{#1~\ref*{#2}}}
\newcommand{\HLinkSuffix}[3]{\hyperref[#2]{#1\ref*{#2}{#3}}}

\newcommand{\figlab}[1]{\label{fig:#1}}
\newcommand{\figref}[1]{\HLink{Figure}{fig:#1}}

\newcommand{\obslab}[1]{\label{observation:#1}}
\newcommand{\obsref}[1]{\HLink{Observation}{observation:#1}}

\newcommand{\itemlab}[1]{\label{item:#1}}
\newcommand{\itemref}[1]{\HLinkSuffix{}{item:#1}{}}

\newcommand{\lemlab}[1]{\label{lemma:#1}}
\newcommand{\lemref}[1]{\HLink{Lemma}{lemma:#1}}%

\newcommand{\remlab}[1]{\label{remark:#1}}
\newcommand{\remref}[1]{\HLink{Remark}{remark:#1}}%

\newcommand{\seclab}[1]{\label{section:#1}}
\newcommand{\secref}[1]{\HLink{Section}{section:#1}}%

\providecommand{\eqlab}[1]{}%
\renewcommand{\eqlab}[1]{\label{equation:#1}}

\newcommand{\remove}[1]{}%
\newcommand{\Set}[2]{\left\{ #1 \;\middle\vert\; #2 \right\}}

\newcommand{\pth}[2][\!]{\mleft({#2}\mright)}%

\newcommand{\pbrcx}[1]{\left[ {#1} \right]}%
\newcommand{\ProbLTR}{\mathbb{P}}%

\newcommand{\Prob}[1]{\mathop{\ProbLTR} \mleft[ #1 \mright]}%
\newcommand{\ProbCond}[2]{\mathop{\ProbLTR}\!\left[%
       #1 \;\middle\vert\; #2 \right]}

\newcommand{\ExChar}{\mathbb{E}}%
\newcommand{\ExSym}{\mathop{\ExChar}}%
\newcommand{\Ex}[2][\!]{\ExSym#1\pbrcx{#2}}

\newcommand{\ceil}[1]{\left\lceil {#1} \right\rceil}

\newcommand{\brc}[1]{\left\{ {#1} \right\}}
\newcommand{\cardin}[1]{\left| {#1} \right|}%

\renewcommand{\th}{th\xspace}

\renewcommand{\Re}{\mathbb{R}}%

\definecolor{blue25}{rgb}{0, 0, 11}

\providecommand{\emphic}[2]{%
   \textcolor{blue25}{%
      \textbf{\emph{#1}}}%
   \index{#2}}

\ifx\PDFBlackWhite\undefined
\else
\renewcommand{\emphic}[2]{\textbf{\emph{#1}}}
\fi
\providecommand{\emphi}[1]{\emphic{#1}{#1}}
\renewcommand{\emphi}[1]{\emphic{#1}{#1}}

\providecommand{\Mh}[1]{#1}%
\usepackage{algorithmic}
\usepackage{caption}
\usepackage{subcaption}

\renewcommand{\Re}{\mathbb{R}}
\newcommand{\dist}{\Mh{\mathrm{d}}}%
\newcommand{\PS}{\Mh{P}}%
\newcommand{\DS}{\EuScript{D}}

\newcommand{\eps}{\varepsilon}

\newcommand{\Term}[1]{\textsf{#1}}
\newcommand{\TermI}[1]{\Term{#1}\index{#1@\Term{#1}}}

\newcommand{\LSH}{\TermI{LSH}\xspace}
\newcommand{\MNIST}{\TermI{MNIST}\xspace}%
\newcommand{\NN}{\TermI{NN}\xspace}
\newcommand{\ANN}{\TermI{ANN}\xspace}
\newcommand{\FANN}{\TermI{FANN}\xspace}

\newcommand{\prb}{\Mh{\varphi}}%
\newcommand{\tldO}{\scalerel*{\widetilde{O}}{j^2}}%

\newcommand{\dsS}{\Mh{\mathcal{S}}}%
\newcommand{\dsQ}{\Mh{\mathcal{Q}}}%

\newcommand{\bucket}{\Mh{H}}

\newcommand{\ball}{\mathbb{B}}%
\newcommand{\nbrY}[2]{\Mh{N}\pth{#1, #2}}%
\newcommand{\nNY}[2]{\Mh{n}\pth{#1, #2}}%
\newcommand{\ballY}[2]{\ball\pth{#1, #2}}%

\newcommand{\pNY}[2]{\Mh{\ProbLTR}\pth{#1, #2}}%

\usepackage[inline]{enumitem}

\newlist{compactitem}{itemize}{5}%
\setlist[compactitem]{topsep=0pt,itemsep=-1ex,partopsep=1ex,parsep=1ex,%
   label=\ensuremath{\bullet}}%

\newlist{compactenumA}{enumerate}{5}%
\setlist[compactenumA]{topsep=0pt,itemsep=-1ex,partopsep=1ex,parsep=1ex,%
   label=(\Alph*)}%

\newlist{compactenuma}{enumerate}{5}%
\setlist[compactenuma]{topsep=0pt,itemsep=-1ex,partopsep=1ex,parsep=1ex,%
   label=(\alph*)}%

\newlist{compactenumI}{enumerate}{5}%
\setlist[compactenumI]{topsep=0pt,itemsep=-1ex,partopsep=1ex,parsep=1ex,%
   label=(\Roman*)}%

\newlist{compactenumi}{enumerate}{5}%
\setlist[compactenumi]{topsep=0pt,itemsep=-1ex,partopsep=1ex,parsep=1ex,%
   label=(\roman*)}%

\newcommand{\ground}{\Mh{\mathcal{U}}}%
\newcommand{\SC}{\Mh{B}}
\newcommand{\SA}{\ensuremath{\Mh{U}}}%
\newcommand{\epsA}{\Mh{\xi}}%

\newcommand{\BadProb}{\Mh{\gamma}}%

\newcommand{\setA}{\Mh{X}}%
\newcommand{\Family}{\Mh{\mathcal{F}}}%
\newcommand{\FamilyA}{\Mh{\mathcal{G}}}%
\newcommand{\FamilyB}{\Mh{\mathcal{H}}}%
\newcommand{\aprxEps}{\Mh{\,\ts\approx_\eps\,\ts}}%

\newcommand{\MS}{\Mh{\mathcal{M}}}%

\newcommand{\OL}{\Mh{\mathcal{O}}}%
\newcommand{\mOL}{\Mh{m^{}_\mathcal{O}}}%
\newcommand{\Tree}{\Mh{T}}%
\newcommand{\obj}{\Mh{o}}%

\newcommand{\Event}{\mathcal{E}}%
\newcommand{\MM}{\Mh{\mathcal{M}}}%

\newcommand{\ts}{\hspace{0.6pt}}%
\newcommand{\Cardin}[1]{%
   \ts\!\!\left\bracevert%
       \!\ts\!%
       \vphantom{#1}%
       #1%
       \!\!\ts%
   \right\bracevert%
   \!\!\ts%
}

\numberwithin{figure}{section}%
\numberwithin{table}{section}%
\numberwithin{equation}{section}%

\ifx\SuppressAuthor\undefined%
\newcommand{\ShowAuthor}[1]{#1}%
\else%
\newcommand{\ShowAuthor}[1]{}%%
\fi

\newcommand{\degC}{\Mh{\mathsf{d}}}%
\newcommand{\DegY}[2]{\Mh{\mathsf{D}}_{#1}\pth{#2}}%
\newcommand{\degY}[2]{\Mh{\mathsf{d}}_{#1}\pth{#2}}%
\newcommand{\degX}[1]{\Mh{\mathsf{d}}\pth{#1}}%

\newcommand{\nSets}{\Mh{g}}%
\begin{document}

\title{Near Neighbor: Who is the Fairest of Them All?}

\author{%
   Sariel Har-Peled%
   \SarielThanks{%
      Work on this paper was partially supported by a NSF AF award
      CCF-1907400.  % Started September 2019
   }%
   \and%
   Sepideh Mahabadi%
   \SepidehThanks{}%
}%
\maketitle

\begin{abstract}
    In this work we study a fair variant of the near neighbor
    problem. Namely, given a set of $n$ points $P$ and a parameter
    $r$, the goal is to preprocess the points, such that given a query
    point $q$, any point in the $r$-neighborhood of the query, i.e.,
    $\ball(q,r)$, have the same probability of being reported as the
    near
    neighbor.
    
    We show that \LSH based algorithms can be made fair, without a
    significant loss in efficiency. Specifically, we show an algorithm
    that reports a point in the $r$-neighborhood of a query $q$ with
    almost uniform probability.  The query time is proportional to
    \begin{math}
        O\bigl( \mathrm{dns}(q.r) \dsQ(n,c) \bigr),
    \end{math}
    and its space is $O(\dsS(n,c))$, where $\dsQ(n,c)$ and $\dsS(n,c)$
    are the query time and space of an \LSH algorithm for
    $c$-approximate near neighbor, and $\mathrm{dns}(q,r)$ is a
    function of the local density around $q$.
    
    Our approach works more generally for sampling uniformly from a
    sub-collection of sets of a given collection and can be used in a
    few other applications. Finally, we run experiments to show
    performance of our approach on real data.
\end{abstract}

%%%%%%%%%%%%%%%%%%%%%%%%%%%%%%%%%%%%%%%%%%%%%%%%%%%%%%%%%%%%%%%%%%%%%%% 
%%%%%%%%%%%%%%%%%%%%%%%%%%%%%%%%%%%%%%%%%%%%%%%%%%%%%%%%%%%%%%%%%%%%%%% 
%%%%%%%%%%%%%%%%%%%%%%%%%%%%%%%%%%%%%%%%%%%%%%%%%%%%%%%%%%%%%%%%%%%%%%% 

\section{Introduction}

Nowadays, many important decisions, such as college admissions,
offering home loans, or estimating the likelihood of recidivism, rely
on machine learning
algorithms. There is a growing concern about the fairness of the
algorithms and creating bias toward a specific population or
feature~\cite{hps-eosl-16, c-fpdis-17, msp-bdras-16,
   kleinberg2017human}. While algorithms are not inherently biased,
nevertheless, they may amplify the already existing biases in the
data.  Hence, this concern has led to the design of fair algorithms
for many different applications, e.g., \cite{donini2018empirical,
   abdlw-rafc-18, pleiss2017fairness, pmlr-v89-chierichetti19a,
   elzayn2019fair, olfat2018convex, chierichetti2017fair,
   backurs2019scalable, bera2019fair, kleindessner2019guarantees}.

Bias in the data used for training machine learning algorithms is a
monumental challenge in creating fair
algorithms~\cite{huang2007correcting, torralba2011unbiased,
   zafar2017fairness, c-fpdis-17}. Here, we are
interested in a somewhat different problem, of handling the bias
introduced by the data-structures used by such algorithms.
Specifically, data-structures may introduce bias in the data stored in
them, and the way they answer queries, because of the way the data is
stored and how it is being accessed. Such a defect leads to selection
bias by the algorithms using such data-structures. It is natural to
want data-structures that do not introduce a selection bias into the
data when handling queries.

The target as such is to derive data-structures that are
bias-neutral. To this end, imagine a data-structure that can return,
as an answer to a query, an item out of a set of acceptable
answers. The purpose is then to return uniformly a random item out of
the set of acceptable outcomes, without explicitly computing the whole
set of acceptable answers (which might be prohibitively expensive).

Several notions of fairness have been studied, including
%and in particular people have
%studies both
{\em group fairness}\footnote{The concept is denoted as statistical
   fairness too, e.g.,~\cite{c-fpdis-17}.}  (where demographics of the
population is preserved in the outcome) and {\em individual fairness}
(where the goal is to treat individuals with similar conditions
similarly) \cite{dwork2012fairness}. In this work, we study the near
neighbor problem from the perspective of individual fairness.

Near Neighbor is a fundamental problem that has applications in many
areas such as machine learning, databases, computer vision,
information retrieval, and many others,
see~\cite{sti-nnmlv-06, ai-nohaa-08} for an overview.  The
problem is formally defined as follows.  Let $(\MS,\dist)$ be a metric
space. Given a set $\PS \subseteq \MS$ of $n$ points and a parameter
$r$, the goal of the \emph{near neighbor} problem is to preprocess
$\PS$, such that for a query point $q \in \MS$, one can report a point
$p \in \PS$, such that $\dist(p,q) \leq r$ if such a point exists. As
all the existing algorithms for the \emph{exact} variant of the
problem have either space or query time that depends exponentially on
the ambient dimension of $\MM$, people have considered the approximate
variant of the problem. In the \emph{$c$-approximate near neighbor}
(\ANN) problem, the algorithm is allowed to report a point $p$ whose
distance to the query is at most $cr$ if a point within distance $r$
of the query exists, for some prespecified constant $c > 1$.

Perhaps the most prominent approach to get an \ANN data structure is
via Locality Sensitive Hashing (\LSH) \cite{im-anntr-98,him-anntr-12},
which leads to sub-linear query time and sub-quadratic space. In
particular, for $\MS=\Re^d$, by using \LSH one can get a query time of
$n^{\rho+o(1)}$ and space $n^{1+\rho+o(1)}$ where for the $L_1$
distance metric $\rho=1/c$ \cite{im-anntr-98,him-anntr-12}, and for
the $L_2$ distance metric $\rho=1/c^2 + o_c(1)$ \cite{ai-nohaa-08}.
The idea of the \LSH method is to hash all the points using several
hash functions that are chosen randomly, with the property that closer
points have a higher probability of collision than the far
points. Therefore, the closer points to a query have a higher
probability of falling into a bucket being probed than far
points. Thus, reporting a random point from a random bucket computed
for the query, produces a distribution that is biased by the distance
to the query: closer points to the query have a higher probability of
being chosen.

\paragraph{When random nearby is better than nearest.}
The bias mentioned above towards nearer points is usually a good
property, but is not always desirable. Indeed, consider the following
scenarios: %\medskip%
\medskip%
\begin{compactenumI}[leftmargin=0.8cm,itemsep=-0.5ex]
    \item The nearest neighbor might not be the best if the input is
    noisy, and the closest point might be viewed as an
    unrepresentative outlier. Any point in the neighborhood might be
    then considered to be equivalently beneficial. This is to some
    extent why $k$-\NN classification \cite{ell-ca-09} is so effective
    in reducing the effect of noise.
    
    % \smallskip
    \item However, $k$-\NN works better in many cases if $k$ is large,
    but computing the $k$ nearest-neighbors is quite expensive if $k$
    is large \cite{haaa-spkp-14}. Computing quickly a random nearby
    neighbor can significantly speed-up such classification.
    
%    \smallskip%
    \item We are interested in annonymizing the query
    \cite{a-u4aql-07}, thus returning a random near-neighbor might
    serve as first line of defense in trying to make it harder to
    recover the query. Similarly, one might want to anonymize the
    nearest-neighbor \cite{qa-eppkn-08}, for applications were we are
    interested in a ``typical'' data item close to the query, without
    identifying the nearest item.
    
%    \smallskip%
    \item If one wants to estimate the number of items with a desired
    property within the neighborhood, then the easiest way to do it is
    via uniform random sampling from the neighborhood. In particular,
    this is useful for density estimation \cite{klk-oknnd-12}.
    
 %   \smallskip%
    \item Another natural application is simulating a random walk in
    the graph where two items are connected if they are in distance at
    most $r$ from each other. Such random walks are used by some graph
    clustering algorithms \cite{hk-curw-01}.
\end{compactenumI}

\subsection{Results}%
\seclab{results}

Our goal is to solve the near-neighbor problem, and yet be fair among
``all the points'' in the neighborhood. We introduce and study the
\emphi{fair near neighbor} problem -- where the goal is to report any
point of $\nbrY{q}{r}$ with uniform distribution.  That is, report a
point within distance $r$ of the query point with probability of
$\pNY{q}{r} = 1/\nNY{q}{r}$, where
$\nNY{q}{r} = \cardin{\nbrY{q}{r}}$.
Naturally, we study the
approximate fair near neighbor problem, where one can hope to get
efficient data-structures. We have the following results:
% Using the new algorithm, we provide several guarantees: \medskip%
\smallskip%
\begin{compactenumI}[leftmargin=0.8cm,itemsep=-0.5ex]
    \item \textbf{Exact neighborhood.}  We present a data structure for reporting
	a neighbor according to an ``almost uniform" distribution
    with space $\dsS(n,c)$, and query time
    $\tldO\bigl( \dsQ(n,c)\cdot \frac{\nNY{q}{cr}}{\nNY{q}{r}}
    \bigr)$, where $\dsS(n,c)$ and $\dsQ(n,c)$ are, respectively, the
    space and query time of the standard $c$-\ANN data structure.
    Note that, the query time of the algorithm might be high if the
    approximate neighborhood of the query is much larger than the
    exact neighborhood.\footnote{As we show, the term
       $\dsQ(n,r)\cdot\frac{\nNY{q}{cr}}{\nNY{q}{r}}$ can also be replaced by
       $\dsQ(n,r)+|\nbrY{q}{cr}\setminus \nbrY{q}{r}|$ which can potentially be
       smaller.}  Guarantees of this data structure hold \emph{with high probability}.
	  See \lemref{final-lem} for the exact statement.
    
%    \smallskip%
    \item \textbf{Approximate neighborhood.}  This formulation reports
    an almost uniform distribution from an approximate neighborhood
    $S$ of the query. We can provide such a data structure that uses
    space $\dsS(n,c)$ and whose query time is $\tldO(\dsQ(n,c))$,
    albeit \emph{in expectation}. See \lemref{approx-neighborhood} for the exact statement.
\end{compactenumI}
\smallskip%
\noindent%
Moreover, the algorithm produces the samples independently of past
queries. In particular, one can assume that an adversary is producing
the set of queries and has full knowledge of the data structure. Even
then the generated samples have the same (almost) uniform guarantees.
Furthermore, we remark that the new sampling
strategy can be embedded in the existing LSH method to achieve
unbiased query results. Finally, we remark that to get a 
distribution that is $(1+\eps)$-uniform (See preliminaries for the definition), 
the dependence of our algorithms on $\eps$ is only $O(\log (1/\eps))$.

Very recently, independent of our work, \cite{aumuller2019fair} also provides a similar definition for the fair near neighbor problem.

\bigskip\noindent%
\textbf{Experiments.} %
Finally, we compare the performance of our algorithm with the
algorithm that uniformly picks a bucket and reports a random point, on
the \MNIST, SIFT10K, and GloVe data sets. Our empirical results show
that while the standard \LSH algorithm fails to fairly sample a point
in the neighborhood of the query, our algorithm produces an empirical
distribution which is much closer to the uniform distribution: it
improves the statistical distance to the uniform distribution by a
significant factor.

%%%%%%%%%%%%%%%%%%%%%%%%%%%%%%%%%%%%%%%%%%%%%%%%%%%%%%%%%%%%%%%%%%%%%%% 
%%%%%%%%%%%%%%%%%%%%%%%%%%%%%%%%%%%%%%%%%%%%%%%%%%%%%%%%%%%%%%%%%%%%%%% 
\subsection{Data-structure: Sampling from a sub-collection of sets}

We first study the more generic problem -- given a collection
$\Family$ of sets from a universe of $n$ elements, a query is a
sub-collection $\FamilyA\subseteq \Family$ of these sets and the goal
is to sample (almost) uniformly from the union of the sets in this
sub-collection. We do this by first sampling a set $X$ in the
sub-collection $\FamilyA$ proportional to the size of the set $X$, and
then sampling an element $x\in X$ uniformly at random. This produces a
distribution on elements in the sub-collection such that any element
$x$ is chosen proportional to its degree $\degX{x}$ (i.e., the number
of sets $X\in \FamilyA$ that $x\in X$). Therefore, we can use
rejection sampling and only report $x$ with probability $1/\degX{x}$.

We can compute the degree by checking if $x\in X$ for all sets
$X\in \FamilyA$ in the collection, which takes time proportional to
$\nSets = |\FamilyA|$. Also because of the rejection sampling, we
might need to repeat this process $O(\deg_{avg})$ times which can be
as large as $\nSets$. This leads to expected runtime $O(\nSets^2 )$ to
generate a single sample, see \lemref{q:exact}.

As a first improvement, we \emph{approximate} the degree using
standard sampling techniques, which can be done in time
$\tldO({\nSets}/{\degX{x}})$. Although this can still be large for
small degrees, however, those small values will also be rejected with
a smaller probability. Using this, we can bound the runtime of a query
by $O( \eps^{-2} \nSets \log n )$, see \lemref{almost-uniform} (the
sampling is $(1\pm \eps)$-uniform), where $n$ is (roughly) the input
size.

Our second improvement, which the authors believe to be quite
interesting, follows by simulating this rejection sampling
directly. This follows by first introducing a heuristic to approximate
the degree, and then shortcutting it to get the desired simulation.
\secref{almost:uniform} describes this modified algorithm. In
particular, one can get uniform sampling with high
probability. Specifically, one can sample uniformly in expected
$O( \nSets \log \BadProb^{-1})$ time, where the sampling succeeds with
probability $\geq 1- \BadProb$. Alternatively, one can sample
$(1\pm\eps)$-uniformly, with the expected running time being
$O\bigl( \nSets \log (n/\eps) \bigr)$. This is a significant
improvement, in both simplicity and dependency on $\eps$, over the
previous scheme.

We also show how to modify that data-structure to handle outliers, as
it is the case for \LSH, as the sampling algorithm needs to ignore
such points once they are reported as a sample.

\bigskip\noindent%
\textbf{Applications.}  Here are a few examples of applications of
such a data-structure (for sampling from a union of sets): \smallskip%
\begin{compactenumA}%[leftmargin=0.8cm]
    \item Given a subset $\setA$ of vertices in the graph, randomly
    pick (with uniform distribution) a neighbor to one of the vertices
    of $\setA$. This can be used in simulating disease spread
    \cite{ke-nem-05}. 
    
    \smallskip%
    \item Here, we use a variant of this data-structure to implement
    the fair \ANN.

    \smallskip%
    \item Uniform sampling for range searching \cite{hqt-irs-14,
       aw-irsr-17, ap-irsra-19}. Indeed, consider a set  of
    points, stored in a data-structure for range
    queries. Using the above, we can support sampling from the points
    reported by several queries, even if the reported answers are not
    disjoint.
\end{compactenumA}
\smallskip%
Being unaware of any previous work on this problem, we believe this
data-structure is of independent interest.

\subsection{Paper organization}%

We describe some basic sampling and approximation tools in
\secref{prelims}. We describe the sampling from union of set
data-structure in \secref{DS}.  The application of the data-structure
to \LSH is described in \secref{s:f:nn}.  The experiments are
described in \secref{experiments}.

%%%%%%%%%%%%%%%%%%%%%%%%%%%%%%%%%%%%%%%%%%%%%%%%%%%%%%%%%%%%%%%%%%%%%%%%%% 5
%%%%%%%%%%%%%%%%%%%%%%%%%%%%%%%%%%%%%%%%%%%%%%%%%%%%%%%%%%%%%%%%%%%%%%%%%% 5
%%%%%%%%%%%%%%%%%%%%%%%%%%%%%%%%%%%%%%%%%%%%%%%%%%%%%%%%%%%%%%%%%%%%%%%%%% 5

\section{Preliminaries}
\seclab{prelims}

\paragraph*{Neighborhood, fair nearest-neighbor, %
   and approximate neighborhood.}

Let $(\MS,\dist)$ be a metric space and let $\PS \subseteq \MS$ be a
set of $n$ points.  Let
$\ball(c,r) = \Set{x\in \MS}{ \dist(c,x)\leq r}$ be the (close) ball
of radius $r$ around a point $c\in \MS$, and let
$\nbrY{c}{r} = \ball(c,r)\cap \PS$ be the \emphi{$r$-neighborhood} of
$c$ in $\PS$. The \emph{size} of the $r$-neighborhood is
$\nNY{c}{r} = \cardin{ \nbrY{c}{r} }$.

\begin{definition}[\FANN]
    Given a data set $\PS \subseteq \MM$ of $n$ points and a parameter
    $r$, the goal is to preprocess $\PS$ such that for a given query
    $q$, one reports each point $p\in \nbrY{q}{r}$ with probability
    $\mu_p$ where $\mu$ is an approximately uniform probability
    distribution:
    $\pNY{q}{r}/ (1+\eps) \leq \mu_p \leq (1+\eps)\pNY{q}{r}$, where
    $\pNY{q}{r} = 1/\nNY{q}{r}$.
\end{definition}

\begin{definition}[\FANN with approximate neighborhood]
    Given a data set $\PS \subseteq \MM$ of $n$ points and a parameter
    $r$, the goal is to preprocess them such that for a given query
    $q$, one reports each point $p\in S$ with probability $\mu_p$
    where $\prb/(1+\eps) \leq \mu_p \leq (1+\eps)\prb$, where $S$ is a
    point set such that
    $\nbrY{q}{r}\subseteq S \subseteq \nbrY{q}{cr}$, and
    $\prb = 1/|S|$.
\end{definition}

\noindent%
\textbf{Set representation.} %
Let $\ground$ be an underlying ground set of $n$ objects (i.e.,
elements). In this paper, we deal with sets of objects. Assume that
such a set $\setA\subseteq \ground$ is stored in some reasonable
data-structure, where one can insert delete, or query an object in
constant time. Querying for an object $\obj\in \ground$, requires
deciding if $\obj \in \setA$. Such a representation of a set is
straightforward to implement using an array to store the objects, and
a hash table.  This representation allows random access to the
elements in the set, or uniform sampling from the set.

If hashing is not feasible, one can just use a standard dictionary
data-structure -- this would slow down the operations by a logarithmic
factor.

\bigskip\noindent%
\textbf{Subset size estimation.}  We need the following standard
estimation tool, \cite[Lemma 2.8]{bhrrs-eeiso-17}.

\begin{lemma}
    \lemlab{est:set}%
    Consider two sets $\SC \subseteq \SA$, where $n = \cardin{\SA}$.
    Let $\epsA, \BadProb \in (0,1)$ be parameters, such that
    $\BadProb < 1/ \log n$. Assume that one is given an access to a
    membership oracle that, given an element $x \in \SA$, returns
    whether or not $x \in \SC$. Then, one can compute an estimate $s$,
    such that
    $(1-\epsA)\cardin{\SC}\leq s \leq (1+\epsA)\cardin{\SC}$, and
    computing this estimates requires
    $O( (n/\cardin{\SC}) \epsA^{-2} \log \BadProb^{-1})$ oracle
    queries. The returned estimate is correct with probability
    $\geq 1 - \BadProb$.
\end{lemma}

\bigskip\noindent%
\textbf{Weighted sampling.}  We need the following standard
data-structure for weighted sampling.

\begin{lemma}\lemlab{ds-tree}
    Given a set of objects $\FamilyB = \brc{ \obj_1, \ldots, \obj_t}$,
    with associated weights $w_1,\ldots, w_t$, one can preprocess them
    in $O(t)$ time, such that one can sample an object out of
    $\FamilyB$.  The probability of an object $\obj_i$ to be sampled
    is $w_i / \sum_{j=1}^t w_j$. In addition the data-structure
    supports updates to the weights. An update or sample operation
    takes $O( \log t)$ time.
\end{lemma}

\begin{proof}
    Build a balanced binary tree $\Tree$, where the objects of
    $\FamilyA$ are stored in the leaves.  Every internal node $u$ of
    $\Tree$, also maintains the total weight $w(u)$ of the objects in
    its subtree. The tree $\Tree$ has height $O( \log t)$, and weight
    updates can be carried out in $O( \log t)$ time, by updating the
    path from the root to the leaf storing the relevant object.

    Sampling is now done as follows -- we start the traversal from the
    root. At each stage, when being at node $u$, the algorithm
    considers the two children $u_1,u_2$. It continues to $u_1$ with
    probability $w(u_1)/ w(u)$, and otherwise it continues into
    $u_2$. The object sampled is the one in the leaf that this
    traversal ends up at.
\end{proof}%

%%%%%%%%%%%%%%%%%%%%%%%%%%%%%%%%%%%%%%%%%%%%%%%%%%%%%%%%%%%%%%%%%%%%%%% 
%%%%%%%%%%%%%%%%%%%%%%%%%%%%%%%%%%%%%%%%%%%%%%%%%%%%%%%%%%%%%%%%%%%%%%% 
%%%%%%%%%%%%%%%%%%%%%%%%%%%%%%%%%%%%%%%%%%%%%%%%%%%%%%%%%%%%%%%%%%%%%%% 

\section{Data-structure: Sampling from the %
   union of sets}
\seclab{DS}

\textbf{The problem.}  Assume you are given a data-structure that
contains a large collection $\Family$ of sets of objects. The sets in
$\Family$ are not necessarily disjoint. The task is to preprocess the
data-structure, such that given a sub-collection
$\FamilyA \subseteq \Family$ of the sets, one can quickly pick
uniformly at random an object from the set
% \begin{equation*}
\begin{math}
    {\textstyle \bigcup} \FamilyA%
    :=%
    \bigcup_{\setA \in \FamilyA} \setA.
\end{math}
% 5\end{equation*}

\bigskip\noindent%
\textbf{Naive solution.}  The naive solution is to take the sets under
consideration (in $\FamilyA$), compute their union, and sample
directly from the union set ${\textstyle \bigcup} \FamilyA$. Our
purpose is to do (much) better -- in particular, the goal is to get a
query time that depends logarithmically on the total size of all sets
in $\FamilyA$.

\subsection{Preprocessing}

% \regVer{\bigskip\noindent}%
% \textbf{Preprocessing.}
For each set $\setA \in \Family$, we build the set representation
mentioned in the preliminaries section.  In addition, we assume that
each set is stored in a data-structure that enables easy random access
or uniform sampling on this set (for example, store each set in its
own array). Thus, for each set $\setA$, and an element, we can decide
if the element is in $\setA$ in constant time.

% \newpage%
\subsection{Uniform sampling via exact degree computation}
\seclab{uniform}

% \paragraph{Query.}
The query is a family $\FamilyA \subseteq \Family$, and define
$m = \Cardin{\FamilyA} := \sum_{\setA \in \FamilyA} \cardin{\setA}$
(which should be distinguished from $\nSets = \cardin{\FamilyA}$ and
from $n= \cardin{\bigcup \FamilyA}$). The \emphi{degree} of an element
$x \in \bigcup \FamilyA$, is the number of sets of $\FamilyA$ that
contains it -- that is,
$\degY{\FamilyA}{x} = \cardin{\DegY{\FamilyA}{x}}$, where
% \begin{equation*}
$ \DegY{\FamilyA}{x}%
=%
\Set{ \setA \in \FamilyA}{ x \in \setA }.  $
% \end{equation*}
The algorithm repeatedly does the following:
\begin{compactenumI}%[leftmargin=0.8cm]
    \smallskip%
    \item \itemlab{s:sample}%
    Picks one set from $\FamilyA$ with probabilities proportional to
    their sizes. That is, a set $\setA \in \FamilyA$ is picked with
    probability $\cardin{\setA} / m$.
    
    \item \itemlab{b:sample}%
    It picks an element $x \in \setA$ uniformly at random.
    
    \item Computes the degree $\degC = \degY{\FamilyA}{x}$.
    
    \item Outputs $x$ and stop with probability $1/\degC$. Otherwise,
    continues to the next iteration.
\end{compactenumI}

\begin{lemma}
    \lemlab{q:exact}%
    Let $n = \cardin{\bigcup\FamilyA}$ and
    $\nSets = \cardin{\FamilyA}$.  The above algorithm samples an
    element $x \in \bigcup \FamilyA$ according to the uniform
    distribution. The algorithm takes in expectation
    $O( \nSets m/n ) = O( \nSets^2 )$ time. The query time is takes
    $O(\nSets^2 \log n)$ with high probability.
\end{lemma}
\begin{proof}
    Let $m = \Cardin{\FamilyA}$.  Observe that an element
    $x \in \bigcup \FamilyA$ is picked by step \itemref{b:sample} with
    probability $\alpha = \degX{x}/m$. The element $x$ is output with
    probability $\beta = 1/ \degX{x}$. As such, the probability of $x$
    to be output by the algorithm in this round is
    $\alpha \beta = 1/ \Cardin{\FamilyA}$. This implies that the
    output distribution is uniform on all the elements of
    $\bigcup \FamilyA$.
    
    The probability of success in a round is $n/m$, which implies that
    in expectation $m/n$ rounds are used, and with high probability
    $O((m/n) \log n)$ rounds. Computing the degree
    $\degY{\FamilyA}{x}$ takes $O( \cardin{\FamilyA})$ time, which
    implies the first bound on the running time. As for the second
    bound, observe that an element can appear only once in each set of
    $\FamilyA$, which readily implies that
    $\degX{y} \leq \cardin{\FamilyA}$, for all
    $y \in \bigcup \FamilyA$.
\end{proof}%

%%%%%%%%%%%%%%%%%%%%%%%%%%%%%%%%%%%%%%%%%%%%%%%%%%%%%%%%%%%%%%%%%%%%%%% 
\subsection{Almost uniform sampling via degree approximation}

The bottleneck in the above algorithm is computing the degree of an
element. We replace this by an approximation.

\begin{defn}
    Given two positive real numbers $x$ and $y$, and a parameter
    $\eps \in (0,1)$, the numbers $x$ and $y$ are
    \emphi{$\eps$-approximation} of each other, denoted by
    $x \aprxEps y$, if $x/(1+\eps) \leq y \leq x(1+\eps)$ and
    $y/(1+\eps) \leq x \leq
    y(1+\eps)$.% {\color{red} two conditions are equivalent}
\end{defn}
In the approximate version, given an item $x \in \bigcup \FamilyA$, we
can approximate its degree and get an improved runtime for the
algorithm.

\begin{lemma}%
    \lemlab{almost-uniform}%
    The input is a family of sets $\Family$ that one can preprocess in
    linear time.  Let $\FamilyA\subseteq\Family$ be a sub-family and
    let $n = \cardin{\bigcup\FamilyA}$, $\nSets = \cardin{\FamilyA}$,
    and $\eps \in (0,1)$ be a parameter.  One can sample an element
    $x \in \bigcup \FamilyA$ with almost uniform probability
    distribution.  Specifically, the probability of an element to be
    output is $\aprxEps 1/n$. After linear time preprocessing, the
    query time is $O\pth{ \nSets \eps^{-2} \log n}$, in expectation,
    and the query succeeds with high probability.
\end{lemma}
\begin{proof}
    Let $m = \Cardin{\FamilyA}$.  Since
    $\degX{x} = \cardin{\DegY{\FamilyA}{x}}$, it follows that we need
    to approximate the size of $\DegY{\FamilyA}{x}$ in
    $\FamilyA$. Given a set $\setA \in \FamilyA$, we can in constant
    time check if $x \in \setA$, and as such decide if
    $\setA \in \DegY{\FamilyA}{x}$. It follows that we can apply the
    algorithm of \lemref{est:set}, which requires
    \begin{math}
        W(x)%
        =%
        O\bigl( \tfrac{\nSets}{ \degX{x}} \eps^{-2} \log n\bigr)
    \end{math}
    time, where the algorithm succeeds with high probability. The
    query algorithm is the same as before, except that it uses the
    estimated degree.

    For $x \in \bigcup \FamilyA$, let $\Event_x$ be the event that the
    element $x$ is picked for estimation in a round, and let
    $\Event_x'$ be the event that it was actually output in that
    round.  Clearly, we have $\ProbCond{\Event_x' }{\Event_x} = 1/d$,
    where $d$ is the degree estimate of $x$. Since
    $d \aprxEps \degX{x}$ (with high probability), it follows that
    $\ProbCond{\Event_x' }{\Event_x} \aprxEps 1/\degX{x}$. Since there
    are $\degX{x}$ copies of $x$ in $\FamilyA$, and the element for
    estimation is picked uniformly from the sets of $\FamilyA$, it
    follows that the probability of any element
    $x \in \bigcup \FamilyA$ to be output in a round is
    \begin{equation*}
        \Prob{\Event_x'}%
        = 
        \ProbCond{\Event_x'}{\Event_x}
        \Prob{\Event_x}
        = 
        \ProbCond{\Event_x'}{\Event_x}
        \frac{ \degX{x}}{ m}  \aprxEps 1/m,
    \end{equation*}
    as $\Event_x' \subseteq \Event_x$.  As such, the probability of
    the algorithm terminating in a round is
    \begin{math}
        \alpha %
        = %
        \sum_{x \in \bigcup \FamilyA} \Prob{\Event_x'}%
        \aprxEps n/m%
        \geq%
        n/2m.
    \end{math}
    As for the expected amount of work in each round, observe that it
    is proportional to
    \begin{equation*}
        W%
        =% 
        \sum_{x \in \bigcup \FamilyA} \Prob{\Event_x} W(x)
        =%
        \sum_{x \in \bigcup \FamilyA}
        \frac{\degX{x}}{m}
        \frac{ \nSets }{ \eps^2 \degX{x}} \log n
        =%
        O\pth{ \frac{ n \nSets}{m} \eps^{-2} \log n}.
    \end{equation*}
    
    Intuitively, since the expected amount of work in each iteration
    is $W$, and the expected number of rounds is $1/\alpha$, the
    expected running time is $O( W / \alpha)$. This argument is not
    quite right, as the amount of work in each round effects the
    probability of the algorithm to terminate in the round (i.e., the
    two variables are not independent). We continue with a bit more
    care -- let $L_i$ be the running time in the $i$\th round of the
    algorithm if it was to do an $i$\th iteration (i.e., think about a
    version of the algorithm that skips the experiment in the end of
    the iteration to decide whether it is going to stop), and let
    $Y_i$ be a random variable that is $1$ if the (original) algorithm
    had not stopped at the end of the first $i$ iterations of the
    algorithm.
    
    By the above, we have that
    $y_i = \Prob{Y_{i}=1} = \ProbCond{Y_{i}=1}{Y_{i-1}
       =1}\Prob{Y_{i-1}=1} \leq (1-\alpha)y_{i-1} \leq (1-\alpha)^i$,
    and $\Ex{L_i} = O(W)$. Importantly, $L_i$ and $Y_{i-1}$ are
    independent (while $L_i$ and $Y_i$ are dependent). We clearly have
    that the running time of the algorithm is
    $O \bigl( \sum_{i=1}^\infty Y_{i-1}L_i \bigr)$ (here, we define
    $Y_0 =1$). Thus, the expected running time of the algorithm is
    proportional to
    \begin{align*}
      \Ex{ \Bigl. \smash{\sum_{i} Y_{i-1}L_i }}%
      &=%
        \sum_{i} \Ex{Y_{i-1}L_i }
        =%
        \sum_{i} \Ex{Y_{i-1}} \Ex{L_i }
        \leq%
        W \sum_{i} y_{i-1}
        \leq%
        W \sum_{i=1}^\infty (1-\alpha)^{i-1}%
        =%
        \frac{W}{\alpha}\\
      &=%
        O( \nSets \eps^{-2} \log n),
    \end{align*}
    because of linearity of expectations, and since $L_i$ and
    $Y_{i-1}$ are independent.
\end{proof}

\begin{remark}%
    \remlab{whp:approx} %
    The query time of \lemref{almost-uniform} deteriorates to
    $O\pth{ \nSets \eps^{-2} \log^2 n}$ if one wants the bound to hold
    with high probability. This follows by restarting the query
    algorithm if the query time exceeds (say by a factor of two) the
    expected running time.  A standard application of Markov's
    inequality implies that this process would have to be restarted at
    most $O( \log n)$ times, with high probability.
\end{remark}

\begin{remark}
    The sampling algorithm is independent of whether or not we fully
    know the underlying family $\Family$ and the sub-family
    $\FamilyA$. This means the past queries do not affect the sampled
    object reported for the query $\FamilyA$. Therefore, the almost
    uniform distribution property holds in the presence of several
    queries and independently for each of them.
\end{remark}

%%%%%%%%%%%%%%%%%%%%%%%%%%%%%%%%%%%%%%%%%%%%%%%%%%%%%%%%%%%%%%%%%%%%%%%
%%%%%%%%%%%%%%%%%%%%%%%%%%%%%%%%%%%%%%%%%%%%%%%%%%%%%%%%%%%%%%%%%%%%%%%
%%%%%%%%%%%%%%%%%%%%%%%%%%%%%%%%%%%%%%%%%%%%%%%%%%%%%%%%%%%%%%%%%%%%%%%
\subsection{Almost uniform sampling via simulation}
\seclab{almost:uniform}

It turns out that one can avoid the degree approximation stage in the
above algorithm, and achieve only a polylogarithmic dependence on 
$\eps^{-1}$. 
To this end, let $x$ be the element picked. We need
to simulate a process that accepts $x$ with probability
$1/\degX{x}$.

We start with the following natural idea for estimating $\degX{x}$ --
probe the sets randomly (with replacement), and stop in the $i$\th
iteration if it is the first iteration where the probe found a set
that contains $x$. If there are $\nSets$ sets, then the distribution
of $i$ is geometric, with probability $p = \degX{x}/\nSets$. In
particular, in expectation, $\Ex{i} = \nSets/\degX{x}$, which implies
that $\degX{x} = \nSets /\Ex{i}$. As such, it is natural to take
$\nSets/i$ as an estimation for the degree of $x$. Thus, to simulate a
process that succeeds with probability $1/\degX{x}$, it would be
natural to return $1$ with probability $i/\nSets$ and $0$
otherwise. Surprisingly, while this seems like a heuristic, it does
work, under the right interpretation, as testified
by the following.%

\begin{lemma}
    \lemlab{first}%
    Assume we have $\nSets$ urns, and exactly $\degC > 0$ of them, are
    non-empty. Furthermore, assume that we can check if a specific urn is
    empty in constant time. Then, there is a randomized algorithm,
    that outputs a number $Y \geq 0$, such that $\Ex{Y}=1/\degC$. The
    expected running time of the algorithm is $O(\nSets/\degC)$.
\end{lemma}%
\begin{proof}
    The algorithm repeatedly probes urns (uniformly at random), until
    it finds a non-empty urn. Assume it found a non-empty urn in the
    $i$\th probe. The algorithm outputs the value $i/\nSets$ and
    stops.
    
    Setting $p = \degC/\nSets$, and let $Y$ be the output of the
    algorithm. we have that
    \begin{equation*}
        \Ex{\bigl. Y}%
        =%
        \sum_{i=1}^\infty \frac{i}{\nSets} (1-p)^{i-1} p
        =%
        \frac{p}{\nSets(1-p)}\sum_{i=1}^\infty i (1-p)^{i} 
        =%
        \frac{p}{\nSets(1-p)} \cdot \frac{1-p}{p^2}
        =%
        \frac{1}{p\nSets}%
        =%
        \frac{1}{\degC},
    \end{equation*}
    using the formula $\sum_{i=1}^\infty ix^i = {x}/{(1-x)^2}$.

    The expected number of probes performed by the algorithm until it
    finds a non-empty urn is $1/p = \nSets/\degC$, which implies that
    the expected running time of the algorithm is $O(\nSets/\degC)$.~
\end{proof}%

The natural way to deploy \lemref{first}, is to run its algorithm to
get a number $y$, and then return $1$ with probability $y$. The
problem is that $y$ can be strictly larger than $1$, which is
meaningless for probabilities. Instead, we backoff by using the value
$y/\Delta$, for some parameter $\Delta$. If the returned value is
larger than $1$, we just treat it at zero.  If the zeroing never
happened, the algorithm would return one with probability
$1/(\degX{x}\Delta)$ -- which we can use to our purposes via,
essentially, amplification. Instead, the probability of success is
going to be slightly smaller, but fortunately, the loss can be made
arbitrarily small by taking $\Delta$ to be sufficiently large.

\begin{lemma}
    \lemlab{second}%
    There are $\nSets$ urns, and exactly $\degC > 0$ of them are not
    empty. Furthermore, assume one can check if a specific urn is
    empty in constant time. Let $\BadProb \in (0,1)$ be a
    parameter. Then one can output a number $Z \geq 0$, such that
    $Z \in [0,1]$, and
    \begin{math}
        \Ex{Z} \in I= \bigl[ \tfrac{1}{\degC\Delta} - \BadProb,
        \tfrac{1}{\degC\Delta} \bigr],\Bigr.%
    \end{math}
    where
    $\Delta = \ceil{\smash{\ln \BadProb^{-1}} } + 4 = \Theta(\log
    \BadProb^{-1})$. The expected running time of the algorithm is
    $O( \nSets/\degC)$.

    Alternatively, the algorithm can output a bit $X$, such that
    $\Prob{X=1} \in I$.
\end{lemma}
\begin{proof}
    We modify the algorithm of \lemref{first}, so that it outputs
    $i/(\nSets\Delta)$ instead of $i/\nSets$. If the algorithm does
    not stop in the first $\nSets\Delta+1$ iterations, then the
    algorithm stops and outputs $0$. Observe that the probability that
    the algorithm fails to stop in the first $\nSets\Delta$
    iterations, for $p = \degC / \nSets$, is
    \begin{math}
        (1-p)^{\nSets\Delta} \leq \exp\pth{ -\frac{\degC}{\nSets}
           \nSets \Delta}%
        \leq%
        \exp( -\degC \Delta) \leq%
        \exp( - \Delta) \ll \BadProb.
    \end{math}

    Let $Z$ be the random variable that is the number output by the
    algorithm. Arguing as in \lemref{first}, we have that
    \begin{math}
        \Ex{Z} \leq 1/(\degC\Delta).
    \end{math}
    More precisely, we have 
    \begin{math}
        \Ex{Z}%
        =%
        \frac{1}{\degC\Delta} - \sum_{i=\nSets\Delta+1}^\infty
        \frac{i}{\nSets \Delta} (1-p)^{i-1} p.
    \end{math}
    Let     
    \begin{align*}
      \sum_{i=\nSets j+1}^{\nSets(j+1)} \frac{i}{\nSets} (1-p)^{i-1} p
      &\leq%
        (j+1)\sum_{i=\nSets j+1}^{\nSets(j+1)}  (1-p)^{i-1} p
        =%
        (j+1)(1-p)^{\nSets j}\sum_{i=0}^{\nSets - 1}  (1-p)^{i} p
      \\
      &\leq%
        (j+1)(1-p)^{\nSets j}
        \leq%
        (j+1) \pth{1-\frac{\degC}{\nSets}}^{\nSets j}
        \leq%
        (j+1) \exp \pth{- \degC j }.
    \end{align*}

    Let $g(j) = \frac{j+1}{\Delta} \exp \pth{- \degC j }$.  We have
    that
    \begin{math}
        \Ex{Z}%
        \geq%
        \frac{1}{\degC \Delta} - \beta,
    \end{math}
    where $\beta = \sum_{j=\Delta}^\infty g(j)$.  Furthermore, for
    $j \geq \Delta$, we have
    \begin{equation*}
        \frac{g(j+1)}{g(j)} %
        =%
        \frac
        {  (j+2) \exp \pth{- \degC (j+1) }}
        {  (j+1) \exp \pth{- \degC j }}%
        \leq %
        \pth{1+\frac{1}{\Delta}} e^{-\degC}
        \leq %
        \frac{5}{4} e^{-\degC}
        \leq%
        \frac{1}{2}.
    \end{equation*}
    As such, we have that
    \begin{equation*}
        \beta%
        =%
        \sum_{j=\Delta}^\infty g(j)%
        \leq%
        2 g(\Delta)%
        \leq%
        2 \frac{\Delta+1}{\Delta} \exp \pth{- \degC \Delta }%
        \leq 
        4 \exp \pth{- \Delta }%
        \leq%
        \BadProb,
    \end{equation*}
    by the choice of value for $\Delta$. This implies that
    $\Ex{Z} \geq 1/(\degC\Delta) - \beta \geq 1/(\degC \Delta) -
    \BadProb$, as desired.

    The alternative algorithm takes the output $Z$, and returns $1$
    with probability $Z$, and zero otherwise.
\end{proof}%

\begin{lemma}%
    \lemlab{almost:uniform:2}%
    The input is a family of sets $\Family$ that one preprocesses in
    linear time.  Let $\FamilyA\subseteq\Family$ be a sub-family and
    let $n = \cardin{\bigcup\FamilyA}$, $\nSets = \cardin{\FamilyA}$,
    and let $\eps \in (0,1)$ be a parameter.  One can sample an
    element $x \in \bigcup \FamilyA$ with almost uniform probability
    distribution.  Specifically, the probability of an element to be
    output is $\aprxEps 1/n$. After linear time preprocessing, the
    query time is $O\pth{ \nSets \log (\nSets/\eps)}$, in expectation,
    and the query succeeds, with high probability (in $\nSets$).
\end{lemma}
\begin{proof}
    The algorithm repeatedly samples an element $x$ using steps
    \itemref{s:sample} and \itemref{b:sample} of the algorithm of
    \secref{uniform}. The algorithm returns $x$ if the algorithm of
    \lemref{second}, invoked with $\BadProb = (\eps/\nSets)^{O(1)}$
    returns $1$. We have that $\Delta = \Theta( \log(\nSets/\eps) )$.
    Let $\alpha = 1/(\degX{x}\Delta)$.  The algorithm returns $x$ in
    this iteration with probability $p$, where
    $p \in [\alpha - \BadProb, \alpha]$.  Observe that
    $\alpha \geq 1/(\nSets\Delta)$, which implies that
    $\BadProb \ll (\eps/4 )\alpha$, it follows that
    $p \aprxEps 1/(\degX{x}\Delta)$, as desired.  The expected running
    time of each round is $O(\nSets/\degX{x})$.

    Arguing as in \lemref{almost-uniform}, this implies that each
    round, in expectation takes $O\pth{ n \nSets / m }$ time, where
    $m = \Cardin{\FamilyA}$. Similarly, the expected number of rounds,
    in expectation, is $O(\Delta m/n)$. Again, arguing as in
    \lemref{almost-uniform}, implies that the expected running time is
    $O(\nSets \Delta ) =O( \nSets \log (\nSets/\eps))$.
\end{proof}%

\begin{remark}%
    \remlab{whp:simul}%
    Similar to \remref{whp:approx}, the query time of
    \lemref{almost:uniform:2} can be made to work with high
    probability with an additional logarithmic factor. Thus with high
    probability, the query time is
    $O\pth{ \nSets \log(\nSets/\eps) \log n}$.
\end{remark}

%%%%%%%%%%%%%%%%%%%%%%%%%%%%%%%%%%%%%%%%%%%%%%%%%%%%%%%%%%%%%%%%%%%%%%%
%%%%%%%%%%%%%%%%%%%%%%%%%%%%%%%%%%%%%%%%%%%%%%%%%%%%%%%%%%%%%%%%%%%%%%%
%%%%%%%%%%%%%%%%%%%%%%%%%%%%%%%%%%%%%%%%%%%%%%%%%%%%%%%%%%%%%%%%%%%%%%%

\subsection{Handling outliers}

Imagine a situation where we have a marked set of outliers $\OL$. We
are interested in sampling from $\bigcup \FamilyA \setminus \OL$.  We
assume that the total degree of the outliers in the query is at most
$\mOL$ for some prespecified parameter $\mOL$. More precisely, we have
$\degY{\FamilyA}{\OL} = \sum_{x \in \OL} \degY{\FamilyA}{x} \leq
\mOL$.

\begin{lemma}
    \lemlab{outliers}%
    The input is a family of sets $\Family$ that one can preprocess in
    linear time. A query is a sub-family $\FamilyA \subseteq \Family$,
    a set of outliers $\OL$, a parameter $\mOL$, and a parameter
    $\eps \in (0,1)$.  One can either
    \begin{compactenumA}
        % \smallskip%
        \item Sample an element $x \in \bigcup \FamilyA \setminus \OL$
        with $\eps$-approximate uniform distribution.  Specifically,
        the probabilities of two elements to be output is the same up
        to a factor of $1\pm \eps$.  \smallskip%
        \item Alternatively, report that
        $\degY{\FamilyA}{\OL} > \mOL$.
    \end{compactenumA}
    % \smallskip%
    The expected query time is $O\pth{ \mOL + \nSets \log (N/\eps)}$, and the
    query succeeds with high probability, where
    $\nSets = \cardin{\FamilyA}$, and $N =
    \Cardin{\Family}$.
\end{lemma}
\begin{proof}
    The main modification of the algorithm of
    \lemref{almost:uniform:2}, is that whenever we encounter an
    outlier (the assumption is that one can check if an element is an
    outlier in constant time), then we delete it from the set $\setA$
    where it was discovered. If we implement sets as arrays, this can
    be done by moving an outlier object to the end of the active
    prefix of the array, and decreasing the count of the active
    array. We also need to decrease the (active) size of the set. If
    the algorithm encounters more than $\mOL$ outliers then it stops
    and reports that the number of outliers is too large.
    
    Otherwise, the algorithm continues as before. The only difference
    is that once the query process is done, the active count (i.e.,
    size) of each set needs to be restored to its original size, as is
    the size of the set. This clearly can be done in time proportional
    to the query time.
\end{proof}%

%%%%%%%%%%%%%%%%%%%%%%%%%%%%%%%%%%%%%%%%%%%%%%%%%%%%%%%%%%%%%%%%%%%%%%% 
%%%%%%%%%%%%%%%%%%%%%%%%%%%%%%%%%%%%%%%%%%%%%%%%%%%%%%%%%%%%%%%%%%%%%%% 
%%%%%%%%%%%%%%%%%%%%%%%%%%%%%%%%%%%%%%%%%%%%%%%%%%%%%%%%%%%%%%%%%%%%%%% 

\section{In the search for a fair near neighbor}
\seclab{s:f:nn}

In this section, we employ our data structure of \secref{DS} to show
the two results on uniformly reporting a neighbor of a query point
mentioned in \secref{results}. First, let us briefly give some
preliminaries on \LSH. We refer the reader to \cite{him-anntr-12} for
further details. Throughout the section, we assume that our metric
space, admits the \LSH data structure.

%%%%%%%%%%%%%%%%%%%%%%%%%%%%%%%%%%%%%%%%%%%%%%%%%%%%%%%%%%%%%%%%%%%%%%%

\subsection{Background on \LSH}

\paragraph{Locality Sensitive Hashing (\LSH).}
Let $\DS$ denote the data structure constructed by \LSH, and let $c$
denote the approximation parameter of \LSH. The data-structure $\DS$
consists of $L$ hash functions $g_1,\ldots,g_L$ (e.g.,
$L \approx n^{1/c}$ for a $c$-approximate \LSH), which are chosen via
a random process and each function hashes the points to a set of
buckets. For a point $p\in \MS$, let $\bucket_i(p)$ be the bucket that
the point $p$ is hashed to using the hash function $g_i$. The
following are standard guarantees provided by the \LSH data structure
\cite{him-anntr-12}.

\begin{lemma}\lemlab{nn-single}
    For a given query point $q$, let $S = \bigcup_i
    \bucket_i(q)$. Then for any point $p\in \nbrY{q}{r}$, we have that
    with a probability of least $1-1/e-1/3$, we have %\smallskip%
    % \begin{compactenumi}
    %     \item
    (i) $p\in S$ and
    % 
    % \smallskip%
    % \item
    (ii) $|S\setminus \ballY{q}{cr}| \leq 3L$, i.e., the number of
    outliers is at most $3L$.
    % \end{compactenumi}
    Moreover, the expected number of outliers in any single bucket
    $H_i(q)$ is at most $1$.
\end{lemma}
Therefore, if we take $t=O(\log n)$ different data structures
$\DS_1,\ldots,\DS_t$ with corresponding hash functions $g^j_i$ to
denote the $i$\th hash function in the $j$\th data structure, we have
the following lemma.

\begin{lemma}
    \lemlab{nn-multi}%
    Let the query point be $q$, and let $p$ be any point in
    $\nbrY{q}{r}$. Then, with high probability, there exists a data
    structure $\DS_j$, such that $p \in S = \bigcup_i \bucket^j_i(q)$
    and $|S\setminus \ballY{q}{cr}|\leq 3L$.
\end{lemma}
By the above, the space used by \LSH is $\dsS(n,c)=\tldO(n\cdot L)$
and the query time is $\dsQ(n,c)=\tldO(L)$.
\subsection{Approximate Neighborhood}

For $t=O(\log n)$, let $\DS_1,\ldots,\DS_t$ be data structures
constructed by \LSH. % The idea is to use \lemref{outliers}.
Let $\Family$ be the set of all buckets in all data structures, i.e.,
\begin{math}
    \Family = \Set{\bigl.\smash{\bucket_i^j}(p)}{ i\leq L , j\leq t,
       p\in \PS}.
\end{math}
For a query point $q$, consider the family $\FamilyA$ of all buckets
containing the query, i.e.,
$\FamilyA = \Set{\smash{H_i^j(q)}}{ i\leq L , j\leq t }$, and thus
$\cardin{\FamilyA} = O(L \log n)$. Moreover, we let $\OL$ to be the
set of outliers, i.e., the points that are farther than $cr$ from $q$.
Note that as mentioned in \lemref{nn-single}, the expected number of
outliers in each bucket of \LSH is at most $1$. Therefore, by
\lemref{outliers}, we immediately get the following result.

\begin{lemma}\lemlab{approx-neighborhood}
    Given a set $\PS$ of $n$ points and a parameter $r$, we can
    preprocess it such that given query $q$, one can report a point
    $p\in S$ with probability $\mu_p$ where
    $\prb/(1+\eps) \leq \mu_p \leq (1+\eps)\prb$, where $S$ is a point
    set such that $\nbrY{q}{r}\subseteq S \subseteq \nbrY{q}{cr}$, and
    $\prb = 1/|S|$. The algorithm uses space $\dsS(n,c)$ and its
    expected query time is $\tldO(\dsQ(n,c)\cdot \log (1/\eps))$. %
\end{lemma}
\begin{proof}
    Let $S=\bigcup\FamilyA\setminus \OL$; by \lemref{nn-multi}, we
    know that $\nbrY{q}{r}\subseteq S \subseteq \nbrY{q}{cr}$, and
    moreover in expectation $\mOL\leq L = \cardin{\FamilyA}$. We apply
    the algorithm of \lemref{outliers}. The runtime of the algorithm
    is in expectation
    $\tldO(\cardin{\FamilyA}\log(1/\eps)) = \tldO(L\cdot \log(1/\eps))
    =\tldO(\dsQ(n,c)\cdot \log(1/\eps))$, and the algorithm produces
    an almost uniform distribution over the points in $S$.
\end{proof}%

\begin{remark}
    For the $L_1$ distance, the runtime of our algorithm is
    $\tldO(n^{(1/c)+o(1)})$ and for the $L_2$ distance, the runtime of
    our algorithm is $\tldO(n^{(1/c^2) + o(1)})$. These matches the
    runtime of the standard \LSH-based near neighbor algorithms up to
    polylog factors.
\end{remark}

\subsection{Exact Neighborhood}
As noted earlier, the result of the previous section only guarantees a
query time which holds in expectation. Here, we provide an algorithm
whose query time holds \emph{with high probability}.  Note that, here
we cannot apply \lemref{outliers} directly, as the total number of
outliers in our data structure might be large with non-negligible
probability (and thus we cannot bound $\mOL$). However, as noted in
\lemref{nn-multi}, with high probability, there exists a subset of
these data structures $J\subseteq [t]$ such that for each $j\in J$,
the number of outliers in $S_j = \bigcup_i \bucket^j_i(q)$ is at most
$3L$, and moreover, we have that
$\nbrY{q}{r}\subseteq \bigcup_{j\in J} S_j$.  Therefore, on a high
level, we make a guess $J'$ of $J$, which we initialize it to
$J'=[t]$, and start by drawing samples from $\FamilyA$; once we
encounter more than $3L$ outliers from a certain data structure
$\DS_j$, we infer that $j\notin J$, update the value of
$J' = J'\setminus \{j\}$, and set the weights of the buckets
corresponding to $\DS_j$ equal to $0$, so that they will never
participate in the sampling process. As such, at any iteration of the
algorithm we are effectively sampling from
$\FamilyA = \Set{ \smash{H_i^j}(q)}{ i\leq L , j\in J'}$.

\bigskip\noindent%
\textbf{Preprocessing.}  We keep $t=O(\log n)$ \LSH data structures
which we refer to as $\DS_1,\ldots,\DS_t$, and we keep the hashed
points by the $i$\th hash function of the $j$\th data structure in the
array denoted by $\bucket^j_i$. Moreover, for each bucket in
$\bucket^j_i$, we store its size $|\bucket^j_i|$.

\bigskip\noindent%
\textbf{Query Processing.}  We maintain the variables $z^j_i$ showing
the weights of the bucket $\bucket_i^j(q)$, which is initialized to
$|\bucket_i^j(q)|$ that is stored in the preprocessing
stage. Moreover, we keep the set of outliers detected from $H_i^j(q)$
in $\OL_i^j$ which is initially set to be empty.  While running the
algorithm, as we detect an outlier in $H_i^j(q)$, we add it to
$\OL_i^j$, and we further decrease $z_i^j$ by one.  Moreover, in order
to keep track of $J'$, for any data structure $\DS_j$, whenever
$\sum_i |\OL_i^j|$ exceeds $3L$, we will ignore all buckets in
$\DS_j$, by setting all corresponding $z^j_i$ to zero.

At each iteration, the algorithm proceeds by sampling a bucket
$\bucket_i^j(q)$ proportional to its weight $z^j_i$, but only among
the set of buckets from those data structures $\DS_j$ for which less
than $3L$ outliers are detected so far, i.e., $j\in J'$. We then
sample a point uniformly at random from the points in the chosen
bucket that have not been detected as an outlier, i.e.,
$H_i^j(q)\setminus\OL_i^j$. If the sampled point is an outlier, we
update our data structure accordingly. Otherwise, we proceed as
in \lemref{almost:uniform:2}.
%%%%%%%%%%%%%%%%%%%%%%%%%%%%%%%%%%%%%%%%%%%%%%%%%%%%%%%%%%%%%%%%%%%%%%%%%% 
% 
% \paragraph{Analysis.}
% 
\begin{definition}[Active data structures and active buckets]
    Consider an iteration $k$ of the algorithm. Let us define the set
    of \emph{active data structures} to be the data structures from
    whom we have seen less than $3L$ outliers so far, and let us
    denote their indices by $J'_k \subseteq [t]$, i.e.,
    $J'_k = \Set{j}{ \sum_i|\OL_j^i|<3L}$.
    
    Moreover, let us define the \emph{active buckets} to be all
    buckets containing the query in these active data structures,
    i.e.,
    $\FamilyA_k = \Set{\smash{H_i^j(q)} }{ i\leq L , j\in J'_k }$.
\end{definition}
\begin{observation}
    \obslab{exact-neighborhood-inclusion}%
    \lemref{nn-multi} implies that with high probability at any
    iteration $k$ of the algorithm
    $\nbrY{q}{r}\subseteq \bigcup\FamilyA_k$.
\end{observation}%

\begin{definition}[active size]
    For an active bucket $\bucket_i^j(q)$, we define its active size
    to be $z_i^j$ which shows the total number of points in the bucket
    that have not yet been detected as an outlier, i.e.,
    $|\bucket_i^j(q)\setminus \OL_i^j|$.
\end{definition}

\begin{lemma}
    \lemlab{approx-dist-overall}%
    Given a set $\PS$ of $n$ points and a parameter $r$, we can
    preprocess it such that given a query $q$, one can report a point
    $p\in \PS$ with probability $\mu_p$, so that there exists a value
    $\rho \in [0,1]$ where
    \begin{compactitem}
        \item For $p\in \nbrY{q}{r}$, we have
        $\frac{\rho}{(1+O(\eps))} \leq \mu_p \leq (1+O(\eps))\rho$. %
        \item For $p\in \nbrY{q}{cr}\setminus \nbrY{q}{r}$, we have
        $\mu_p \leq (1+O(\eps))\rho$.
        \item For $p\notin \nbrY{q}{cr}$, we have $\mu_p = 0$.
    \end{compactitem}
    The space used is $\tldO(\dsS(n,c))$ and the query time is
    $\tldO\bigl( \dsQ(n,c)\cdot \log(1/\eps))$ with high probability. %
\end{lemma}
\begin{proof}
    First note that the algorithm never outputs an outlier, and thus
    the third item is always satisfied. Next, let $K$ be a random
    variable showing the number of iterations of the algorithm, and
    for an iteration $k$, define the random variable
    $M_k = \nbrY{q}{cr} \cap \bigcup \FamilyA_k$ as the set of
    non-outlier points in the set of active buckets.  Conditioned on
    $K=k$, by \lemref{almost:uniform:2}, we know that the distribution
    of the output is almost uniform on $M_k$.  Moreover, we know that
    for all $k$ we have $M_k \subseteq M_{k-1}$, and that by
    \obsref{exact-neighborhood-inclusion}, $\nbrY{q}{r}\subseteq
    M_k$. Therefore, for all points in $\nbrY{q}{r}$ their probability
    of being reported as the final output of the algorithm is equal,
    and moreover, for all points in
    $\nbrY{q}{cr}\setminus \nbrY{q}{r}$, their probability of being
    reported is lower (as at some iteration, some of these points
    might go out of the set of active buckets). This proves the
    probability condition.
    
    \iffalse To be more formal, let $D$ be a random variable showing
    the state of the data structures (which includes the state of
    $Outliers_i^j$). Knowing the value of $S$, the value
    $M = M(S)=\nbrY{q}{cr}\cap B(S)$ is uniquely identified. Therefore
    using \lemref{approx-dist-uniform}, the probability that a point
    $p\in \nbrY{q}{r}$ is the output of the algorithm is
    \begin{align*}
      \sum_{k , s} \mathbb{P}[K=k , S = s]\cdot
      \mathbb{P}[\text{output is }p | K=k , S = s] \leq
      \sum_{k , s} \mathbb{P}[K=k , S = s]
      \frac{1+O(\eps)}{M(s)} = (1+O(\eps)) \rho
    \end{align*}
    
    where we defined $\rho$ to be
    $\sum_{k , s} \mathbb{P}[K=k , S = s] / M(s)$. Note that $\rho$ is
    independent of the point $p$. Similarly we can also lower bound it
    by $\rho/(1+O(\eps))$, again using
    \lemref{approx-dist-uniform}. For the points
    $p\in \nbrY{q}{cr}\setminus \nbrY{q}{r}$, if the $p\in M(s)$, then
    the probability is at most $(1+O(\eps))/M(s)$ and it is $0$
    otherwise. This completes the proof of the lemma.  \fi
    
    To bound the query time, let us consider the iterations where the
    sampled point $p$ is an outlier, and not an outlier,
    separately. The total number of iterations where an outlier point
    is sampled is at most $3L \cdot t = \tldO(L)=\tldO(\dsQ(n,c))$ for
    which we only pay $\tldO(1)$ cost. For non-outlier points, their
    total cost can be bounded using \lemref{almost:uniform:2} and \remref{whp:simul} by
    $\tldO(\cardin{\FamilyA_1}\log(1/\eps)) = \tldO(L\cdot \log(1/\eps))
    =\tldO(\dsQ(n,c)\cdot \log(1/\eps))$.
\end{proof}%

\begin{lemma}
    \lemlab{final-lem}%
    Given a set $\PS$ of $n$ points and a parameter $r$, we can
    preprocess it such that given a query $q$, one can report a point
    $p\in S$ with probability $\mu_p$ where $\mu$ is an approximately
    uniform probability distribution:
    $\prb/ (1+\eps) \leq \mu_p \leq \prb(1+\eps)$, where
    $\prb = 1/ |\nbrY{q}{r}|$. The algorithm uses space $\dsS(n,c)$
    and has query time of
    $\tldO\bigl( \dsQ(n,c)\cdot
    \frac{|\nbrY{q}{cr}|}{|\nbrY{q}{r}|}\cdot \log(1/\eps) \bigr)$
    with high probability. %
\end{lemma}
\begin{proof}
    We run Algorithm of \lemref{approx-dist-overall}, and while its
    output is outside of $\nbrY{q}{r}$, we ignore it and run the
    algorithm again. By \lemref{approx-dist-overall}, the output is
    guaranteed to be almost uniform on $\nbrY{q}{r}$. Moreover, by
    \lemref{approx-dist-overall}, and because with high probability,
    we only need to run the algorithm
    $\tldO(\frac{|\nbrY{q}{cr}|}{|\nbrY{q}{r}|})$ times, we get the
    desired bound on the query time.
\end{proof}%

%%%%%%%%%%%%%%%%%%%%%%%%%%%%%%%%%%%%%%%%%%%%%%%%%%%%%%%%%%%%%%%%%%%%%%% 
%%%%%%%%%%%%%%%%%%%%%%%%%%%%%%%%%%%%%%%%%%%%%%%%%%%%%%%%%%%%%%%%%%%%%%% 
%%%%%%%%%%%%%%%%%%%%%%%%%%%%%%%%%%%%%%%%%%%%%%%%%%%%%%%%%%%%%%%%%%%%%%% 
% 
\section{Experiments}%
\seclab{experiments}%

In this section, we consider the task of retrieving a random point
from the neighborhood of a given query point, and evaluate the
effectiveness of our proposed algorithm empirically on real data
sets.

\bigskip\noindent%
\textbf{Data set and Queries.}  
We run our experiments on three datasets that  are standard benchmarks in the context of Nearest Neighbor algorithms (see \cite{dataset})
\begin{compactenumI}
\item Our first data set contains a random subset of
10K points in the \MNIST training data set
\cite{lecun1998gradient}\footnote{The dataset is available here:
   \href{http://yann.lecun.com/exdb/mnist/}{http://yann.lecun.com/exdb/mnist/}}. The
full data set contains 60K images of hand-written digits, where
each image is of size $28$ by $28$.  For the query, we use a random subset of
$100$ (out of 10K) images of the \MNIST test data set.  Therefore,
each of our points lie in a $784$ dimensional Euclidean space and each
coordinate is in $[0,255]$.
\item Second, we take SIFT10K image descriptors that contains 10K 128-dimensional points as data set and 100 points as queries \footnote{The dataset if available here: \href{http://corpus-texmex.irisa.fr/}{http://corpus-texmex.irisa.fr/}}.
\item Finally, we take a random subset of 10K words from the GloVe data set \cite{pennington2014glove} and a random subset of 100 words as our query. GloVe is a data set of 1.2M word embeddings in 100-dimensional space and we further normalize them to unit norm.
\end{compactenumI}
We use the $L_2$ Euclidean distance to
measure the distance between the points.

\bigskip\noindent%
\textbf{\LSH data structure and parameters.} We use the locality
sensitive hashing data structure for the $L_2$ Euclidean distance
\cite{ai-nohaa-08}. That is, each of the $L$ hash functions $g_i$, is
a concatenation of $k$ unit hash functions
$h_i^1\oplus\cdots\oplus h_i^k$. Each of the unit hash functions
$h_i^j$ is chosen by selecting a point in a random direction (by
choosing every coordinate from a Gaussian distribution with parameters
$(0,1)$). Then all the points are projected onto this one dimensional
direction. Then we put a randomly shifted one dimensional grid of
length $w$ along this direction. The cells of this grid are considered
as buckets of the unit hash function. 
For tuning the parameters of \LSH, we follow the method described in \cite{diim-lshsb-04}, and the manual of E2LSH library \cite{E2LSH}, as follows.

For MNIST, the average distance of a query to its nearest neighbor in the our
data set is around $4.5$. Thus we choose the near neighbor radius
$r = 5$. Consequently, as we observe, the $r$-neighborhood of at least
half of the queries are non-empty.  As suggested in
\cite{diim-lshsb-04} to set the value of $w=4$, we tune it between $3$
and $5$ and set its value to $w=3.1$.  We tune $k$ and $L$ so that the
false negative rate (the near points that are not retrieved by \LSH)
is less than $10\%$, and moreover the cost of hashing (proportional to
$L$) balances out the cost of scanning.  We thus get $k=15$ and
$L=100$. This also agrees with the fact that $L$ should be roughly square root of the total number of points. 
Note that we use a single \LSH data structure as opposed to
taking $t=O(\log n)$ instances. We use the same method for the other two data sets.
For SIFT, we use $R=255$, $w=4$, $k=15$, $L=100$, and for GloVe we use $R=0.9$, $w=3.3$, $k=15$, and $L=100$.

\bigskip\noindent%
\textbf{Algorithms.} Given a query point $q$, we retrieve all $L$
buckets corresponding to the query. We then implement the following
algorithms and compare their performance in returning a neighbor of
the query point.
\smallskip%
\begin{compactitem}[leftmargin=0.5cm]
    \item \textbf{Uniform/Uniform}: Picks bucket uniformly at random
    and picks a random point in bucket. %
    \item \textbf{Weighted/Uniform}: Picks bucket according to its
    size, and picks uniformly random point inside bucket.

    \item \textbf{Optimal}: Picks bucket according to size, and then
    picks uniformly random point $p$ inside bucket. Then it computes
    $p$'s degree \emph{exactly} and rejects $p$ with probability
    $1-1/deg(p)$.
    \item \textbf{Degree approximation}: Picks bucket according to
    size, and picks uniformly random point $p$ inside bucket. It
    approximates $p$'s degree and rejects $p$ with probability
    $1-1/deg'(p)$.
\end{compactitem}

\bigskip\noindent%
\textbf{Degree approximation method.} We use the algorithm of
\secref{almost:uniform} for the degree approximation: we implement a
variant of the sampling algorithm which repeatedly samples a bucket
uniformly at random and checks whether $p$ belongs to the bucket. If
the first time this happens is at iteration $i$, then it outputs the
estimate as $deg'(p)=L/i$.

\bigskip\noindent%
\textbf{Experiment Setup.}  In order to compare the performance of
different algorithms, for each query $q$, we compute $M(q)$: the set
of neighbors of $q$ which fall to the same bucket as $q$ by at least
one of the $L$ hash functions. Then for $100 |M(q)|$ times, we draw a
sample from the neighborhood of the query, using all four
algorithms. We compare the empirical distribution of the reported
points on $|M(q)|$ with the uniform distribution on it. More
specifically, we compute the total variation distance (statistical
distance)\footnote{For two discrete distributions $\mu$ and $\nu$ on a
   finite set $X$, the total variation distance is
   $\frac{1}{2}\sum_{x\in X} |\mu(x)-\nu(x)|$.}  to the uniform
distribution.  We repeat each experiment $10$ times and report the
average result of all $10$ experiments over all $100$ query points.

\bigskip\noindent%
\textbf{Results.}  \figref{res} shows the comparison between all four
algorithms. To compare their performance, we compute the total
variation distance of the empirical distribution of the algorithms to
the uniform distribution.  For the tuned parameters ($k=15$ ,
$L=100$), our results are as follows.  For MNIST, we see that our
proposed degree approximation based algorithm performs only $2.4$
times worse than the optimal algorithm, while we see that other
standard sampling methods perform $6.6$ times and $10$ times worse
than the optimal algorithm. For SIFT, our algorithm performs only
$1.4$ times worse than the optimal while the other two perform $6.1$
and $9.7$ times worse. For GloVe, our algorithm performs only $2.7$
times worse while the other two perform $6.5$ and $13.1$ times worse
than the optimal algorithm.

Moreover, in order get a different range of degrees and show that our algorithm 
works well for those cases, we further vary the parameters $k$ and $L$ of LSH.
More precisely, to get higher
ranges of the degrees, first we
decrease $k$ (the number of unit hash functions used in each of the
$L$ hash function); this will result in more collisions. Second, we
increase $L$ (the total number of hash functions). These are two ways
to increase the degree of points. For example for the MNIST data set, the
above procedure increases the degree range from $[1,33]$ to $[1,99]$.
\begin{figure}[!h]%
    \centering
    \begin{subfigure}{.5\textwidth}
        \centering
        \includegraphics[width=1.01\linewidth]{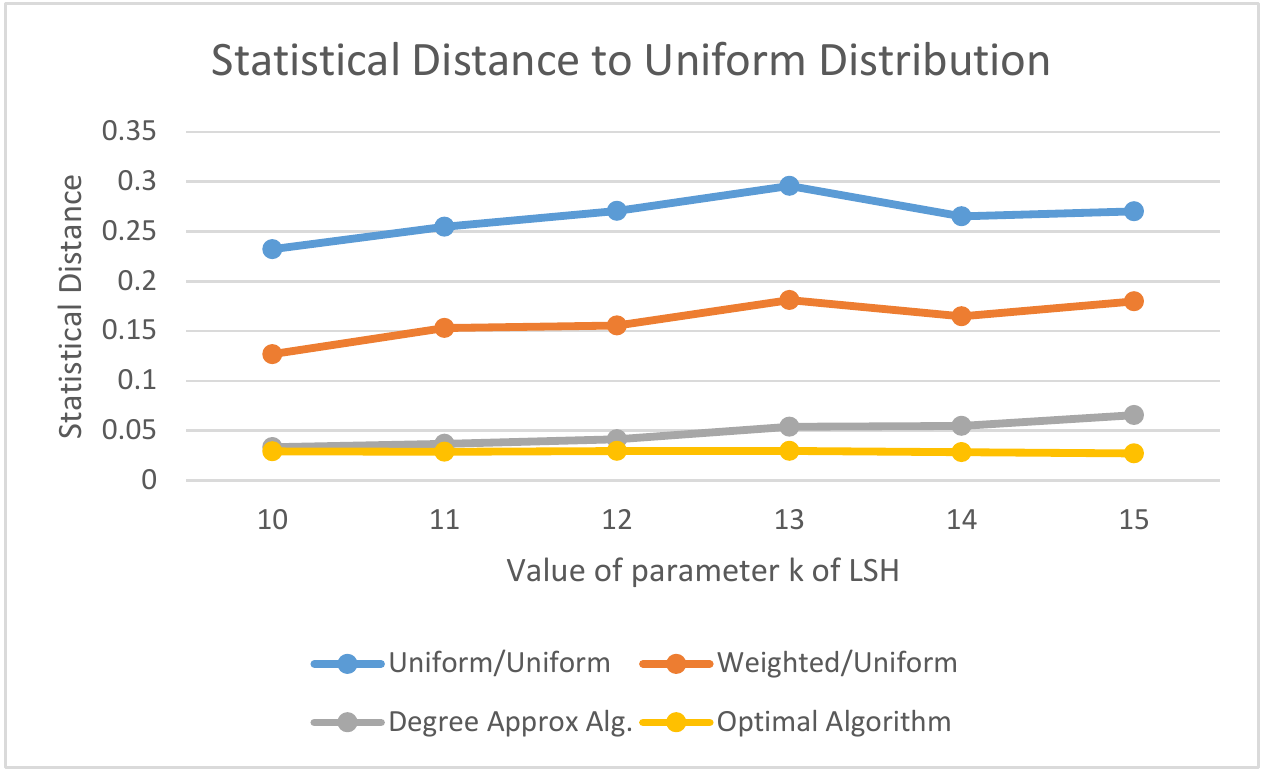}
        \caption{MNIST, varying the parameter $k$ of \LSH}
        % \label{fig:sub1}
    \end{subfigure}%
    \begin{subfigure}{.5\textwidth}
        \centering
        \includegraphics[width=1.01\linewidth]{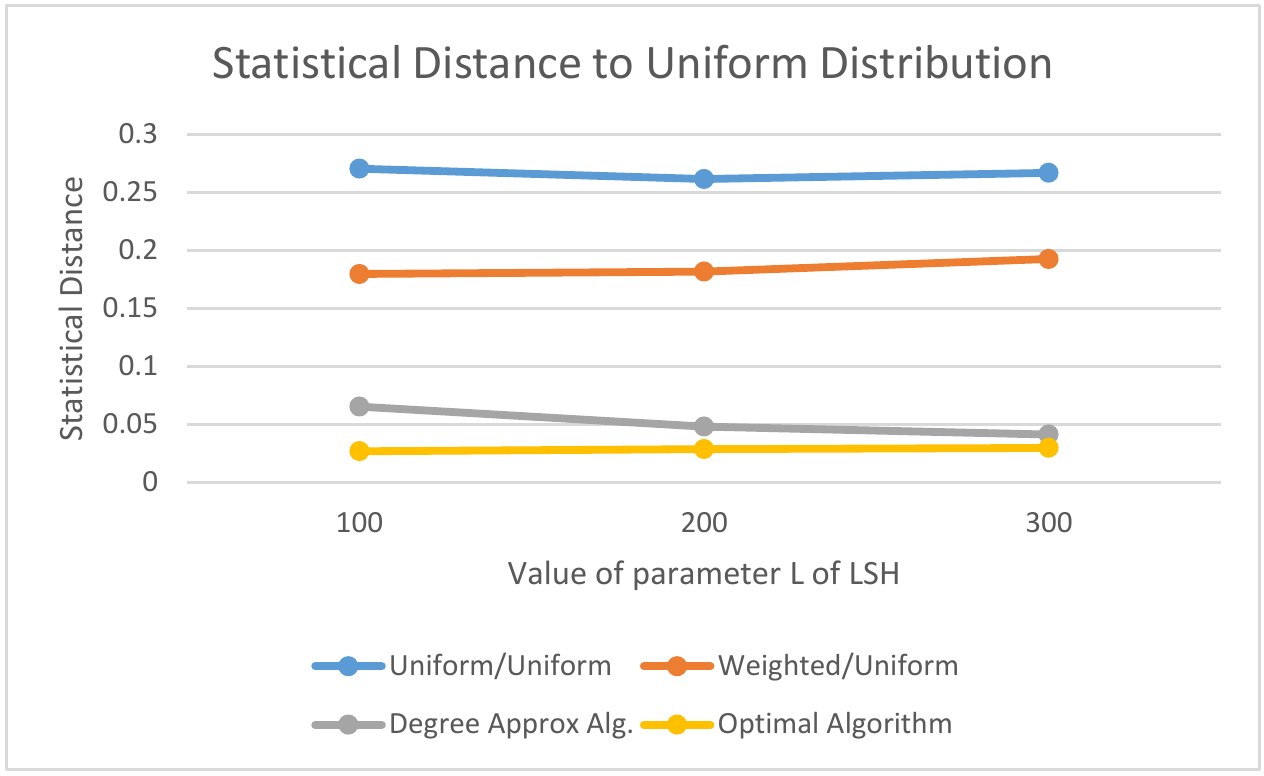}
        \caption{MNIST, varying the parameter $L$ of \LSH}
        % \label{fig:sub2}
    \end{subfigure}

    \begin{subfigure}{.5\textwidth}
        \centering
        \includegraphics[width=1.01\linewidth]{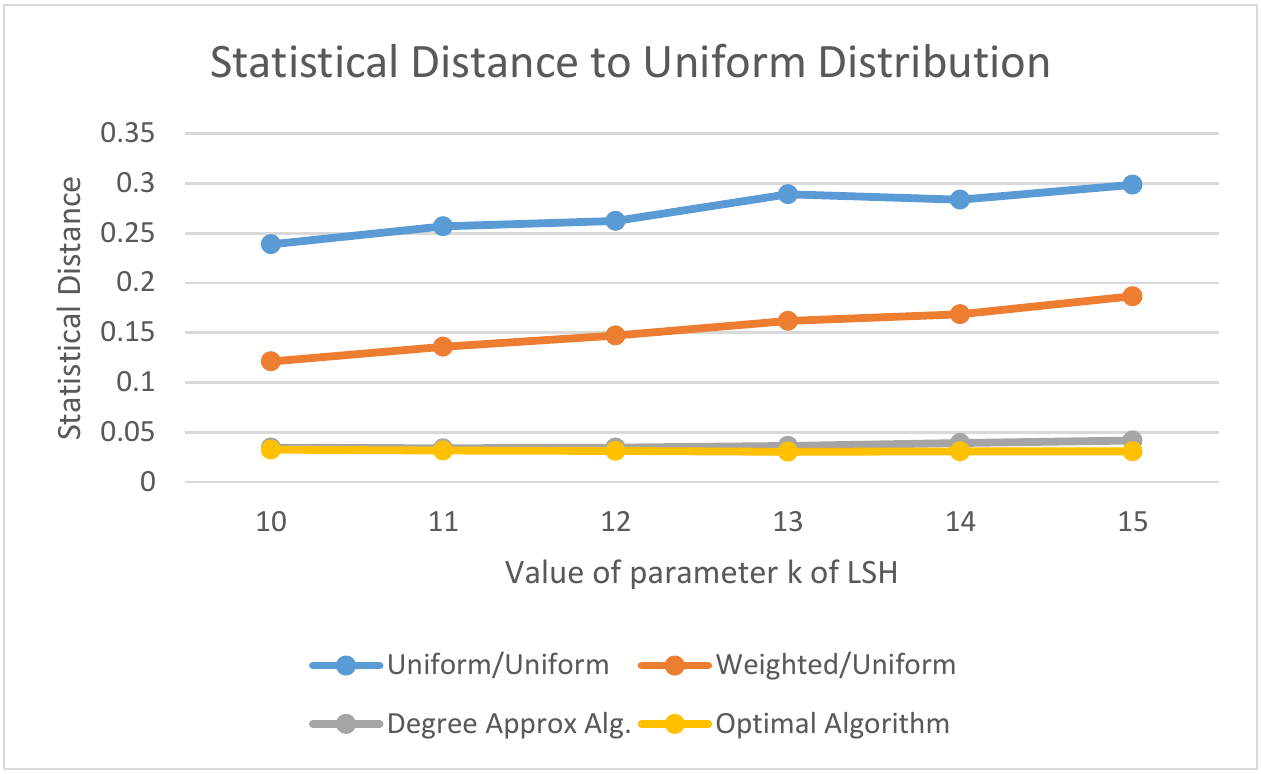}
        \caption{SIFT, varying the parameter $k$ of \LSH}
        % \label{fig:sub1}
    \end{subfigure}%
    \begin{subfigure}{.5\textwidth}
        \centering
        \includegraphics[width=1.01\linewidth]{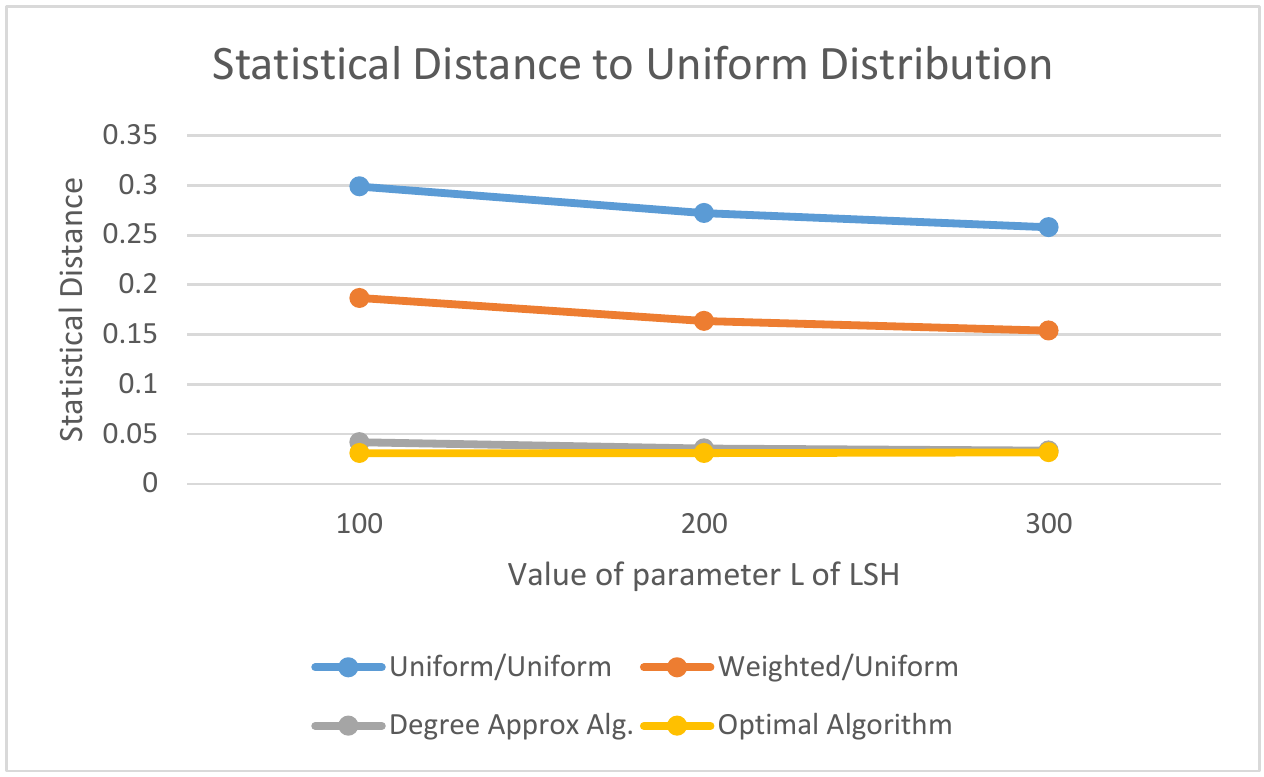}
        \caption{SIFT, varying the parameter $L$ of \LSH}
        % \label{fig:sub2}
    \end{subfigure}

    \begin{subfigure}{.5\textwidth}
        \centering
        \includegraphics[width=1.01\linewidth]{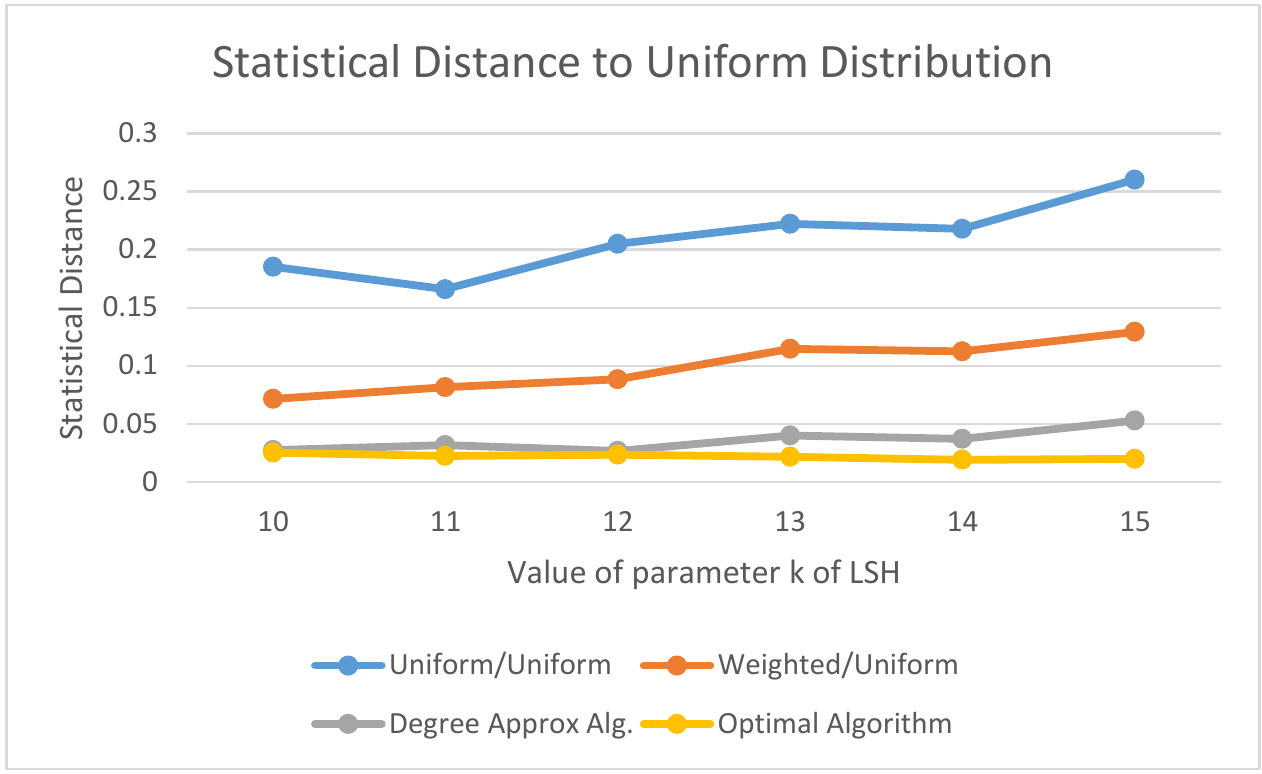}
        \caption{GloVe, varying the parameter $k$ of \LSH}
        % \label{fig:sub1}
    \end{subfigure}%
    \begin{subfigure}{.5\textwidth}
        \centering
        \includegraphics[width=1.01\linewidth]{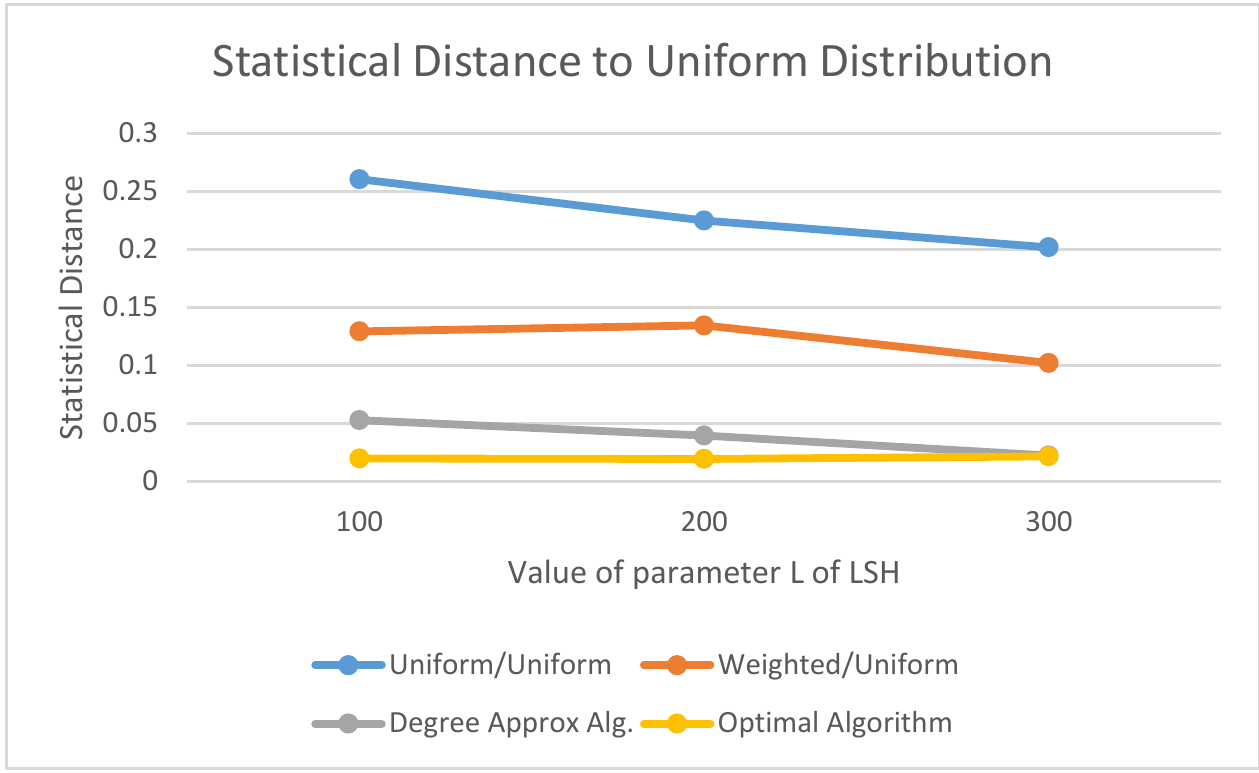}
        \caption{GloVe, varying the parameter $L$ of \LSH}
        % \label{fig:sub2}
    \end{subfigure}

    \caption{Comparison of the performance of the four algorithms is
       measured by computing the statistical distance of their
       empirical distribution to the uniform distribution.}
    \figlab{res}
\end{figure}

\bigskip\noindent%
\textbf{Query time discussion.}  As stated in the experiment setup, in
order to have a meaningful comparison between distributions, in our
code, we retrieve a random neighbor of each query $100m$ times, where
m is the size of its neighborhood (which itself can be as large as
1000). We further repeat each experiment 10 times. Thus, every query
might be asked upto $10^6$ times. This is going to be costly for the
optimal algorithm that computes the degree exactly. Thus, we use the
fact that we are asking the same query many times and preprocess the
exact degrees for the optimal solution. Therefore, it is not
meaningful to compare runtimes directly. Thus we run the experiments
on a smaller size dataset to compare the runtimes of all the four
approaches: For $k=15$ and $L=100$, our sampling approach is twice
faster than the optimal algorithm, and almost five times slower than
the other two approaches. However, when the number of buckets (L)
increases from 100 to 300, our algorithm is 4.3 times faster than the
optimal algorithm, and almost 15 times slower than the other two
approaches.

\bigskip\noindent%
\textbf{Trade-off of time and accuracy.}  We can show a trade-off
between our proposed sampling approach and the optimal. For the MNIST
data set with tuned parameters ($k=15$ and $L=100$), by asking twice
more queries (for degree approximation), the solution of our approach
improves from 2.4 to 1.6, and with three times more, it improves to
1.2, and with four times more, it improves to 1.05.  For the SIFT data
set (using the same parameters), using twice more queries, the
solution improves from 1.4 to 1.16, and with three times more, it
improves to 1.04, and with four times more, it improves to 1.05.  For
GloVe, using twice more queries, the solution improves from 2.7 to
1.47, and with three times more, it improves to 1.14, and with four
times more, it improves to 1.01.

\section{Acknowledgement}
The authors would like to thank Piotr Indyk for the helpful discussions about the modeling and experimental sections of the paper.

%*flatex input: [./fair_nn.bbl]
\newcommand{\etalchar}[1]{$^{#1}$}
 \providecommand{\CNFX}[1]{ {\em{\textrm{(#1)}}}}

% flatex input end: [./fair_nn.bbl]
%FLATEX-REM:\bibliography{fair_nn}
%\bibliography{shortcuts,geometry}

\end{document}